\documentclass{article}

\usepackage{microtype}
\usepackage{graphicx}
\usepackage{subcaption}
\usepackage{booktabs} 
\usepackage{multirow}

\usepackage{hyperref}

\usepackage[preprint]{icml2026}

\usepackage{amsmath}
\usepackage{amssymb}
\usepackage{mathtools}
\usepackage{amsthm}
\usepackage{adjustbox}

\usepackage[capitalize,noabbrev]{cleveref}

\theoremstyle{plain}
\newtheorem{theorem}{Theorem}[section]

\theoremstyle{definition}

\theoremstyle{remark}

\usepackage{todonotes}

\usepackage[inline]{enumitem}
\usepackage{tabularx}
\usepackage{dsfont}
\usepackage{svg}
\usepackage{import}    
\usepackage{xcolor}    

\DeclareMathOperator{\DLGN}{\text{DLGN}}
\DeclareMathOperator{\LLNN}{\text{LLNN}}
\DeclareMathOperator{\LIGHT}{\text{LIGHT}}
\DeclareMathOperator{\WARP}{\text{WARP}}
\DeclareMathOperator{\DWN}{\text{DWN}}
\DeclareMathOperator{\LNN}{\text{LNN}}

\DeclareMathOperator*{\argmin}{\mathrm{arg\,min}}
\DeclareMathOperator*{\sgn}{\mathrm{sgn}}

\renewcommand{\vec}[1]{\boldsymbol{#1}}
\newcommand{\mat}[1]{\boldsymbol{#1}}
\newcommand{\basis}[2]{\phi_{#2}^{#1}}
\newcommand{\lut}[1]{t^{(#1)}}
\newcommand{\veclut}[1]{\vec{t}^{(#1)}}

\newcommand{\btp}{\varphi}  

\crefname{equation}{Eq.}{Eqs.}
\crefname{section}{Sec.}{Secs.}
\crefname{figure}{Fig.}{Figs.}
\crefname{table}{Tab.}{Tabs.}
\crefname{theorem}{Theorem}{Theorems}

\icmltitlerunning{WARP Logic Neural Networks}

\begin{document}
\twocolumn[
  \icmltitle{WARP Logic Neural Networks}



  \icmlsetsymbol{equal}{*}

  \begin{icmlauthorlist}
    \icmlauthor{Lino Gerlach}{equal,prin}
    \icmlauthor{Thore Gerlach}{equal,esa}
    \icmlauthor{Liv Våge}{prin}
    \icmlauthor{Elliott Kauffman}{prin}
    \icmlauthor{Isobel Ojalvo}{prin}
  \end{icmlauthorlist}

  \icmlaffiliation{prin}{Princeton University}
  \icmlaffiliation{esa}{European Space Agency, Advanced Concepts Team}

  \icmlcorrespondingauthor{Lino Gerlach}{lg0508@princeton.edu}
  \icmlcorrespondingauthor{Thore Gerlach}{thore.gerlach@esa.int}

  \icmlkeywords{Fast Machine Learning, Weightless Neural Networks, FPGA, Logic Gates}

  \vskip 0.3in
]



\printAffiliationsAndNotice{\icmlEqualContribution}
\begin{abstract}
Fast and efficient AI inference is increasingly important, and recent models that directly learn low-level logic operations have achieved state-of-the-art performance.
However, existing logic neural networks incur high training costs, introduce redundancy or rely on approximate gradients, which limits scalability.
To overcome these limitations, we introduce \underline{WA}lsh \underline{R}elaxation for \underline{P}robabilistic~(WARP) logic neural networks---a novel gradient-based framework that efficiently learns combinations of hardware-native logic blocks.
We show that WARP yields the most parameter-efficient representation for exactly learning Boolean functions and that several prior approaches arise as restricted special cases.
Training is improved by introducing learnable thresholding and residual initialization, while we bridge the gap between relaxed training and discrete logic inference through stochastic smoothing.
Experiments demonstrate faster convergence than state-of-the-art baselines, while scaling effectively to deeper architectures and logic functions with higher input arity.
\end{abstract}

\section{Introduction}
\label{sec:introduction}



Deep learning~\cite{lecun2015deep} has become the standard for a wide range of tasks in science,
but its success comes with heavy computational costs in both training and inference. This restricts deployability in many real-world settings---particularly in domains where ultra-fast inference is critical, such as healthcare~\cite{bacellar2025nanoml}, particle physics~\cite{Aarrestad_2021}, gravitational wave astronomy~\cite{martins2025improving}, and quantum computing~\cite{bhat2025machine}.
These constraints have motivated substantial research into the development of models that maintain predictive accuracy while improving computational efficiency.
Prominent among such efforts are model compression techniques, which encompass approaches for inducing sparsity~\cite{sung2021training,hoefler2021sparsity}, pruning~\cite{lin2018accelerating,liu2022unreasonable}, and reducing numerical precision via weight quantization~\cite{sun2024gradient,gholami2022survey,chmielneural}.

Nevertheless, such approaches do not directly address the intrinsic computational cost of numerical multiplication.
To overcome this limitation, multiplication-free architectures have been proposed, including binary neural networks~\cite{hubara2016binarized,qin2020binary} and other bit-level summation models~\cite{chen2020addernet,elhoushi2021deepshift,nguyen2024bold}.
However, efficient inference on digital hardware requires mapping abstract computations into executable logic, which incurs significant overhead.
\begin{figure*}[t!]
    \centering
    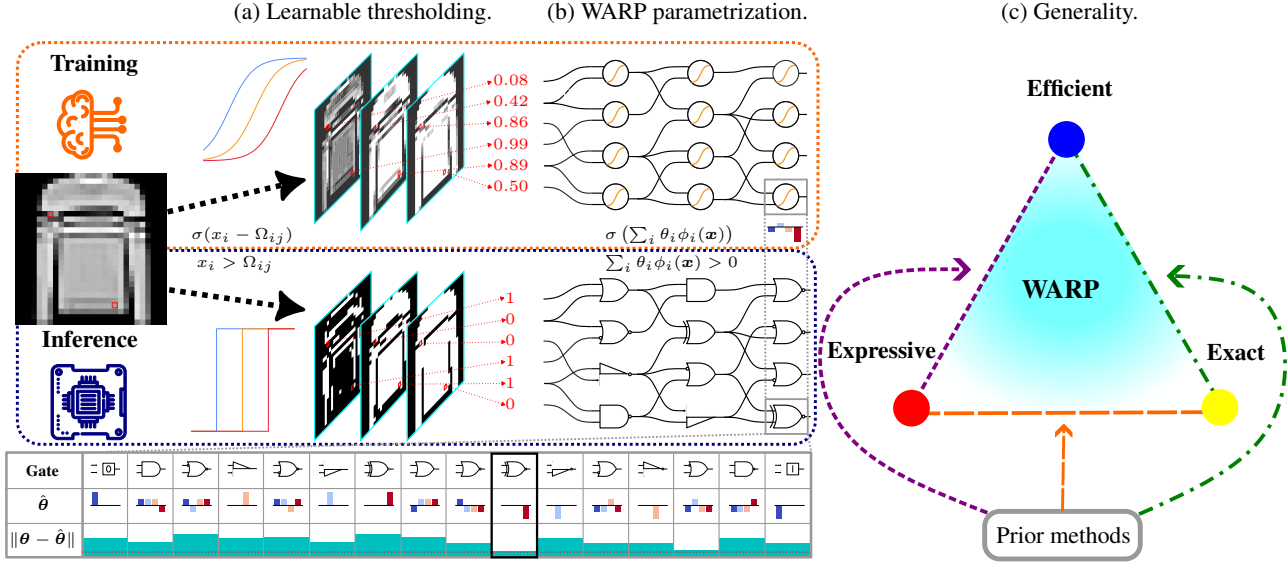
    \caption{Overview of proposed $\WARP$ LNNs enabling differentiable learning of Boolean functions. During training, single neurons are continuously relaxed, while being replaced by a logic function to enable fast and efficient inference. \textbf{(a)} We propose learnable thresholding with relaxing input binarization, which leads to improved performance. \textbf{(b)} For parameterizing neurons, we introduce $\WARP$ which is based on the Walsh-Hadamard transform. This leads to exponentially less parameters ($16$ to $4$ for $n=2$) and deployable logic structures are obtained by optimal thresholding. \textbf{(c)} $\WARP$ is maximally parameter-efficient, does not rely on approximate gradients (exact) and is fully expressive, describing the ability to represent any Boolean function (see~\cref{theo:generality}). While previous methods mostly satisfy only two of these properties (parameter inefficient~\cite{petersen2022deep}, inexact through approximations~\cite{bacellar2024differentiable} or not fully expressive~\cite{hoang2025kanel,andronic2025neuralut}), they can be deduced from $\WARP$ as special cases.}
    \label{fig:overview}
\end{figure*}

Within the domain of multiplication-free models, Logic Neural Networks~($\LNN$s), also denoted as Weightless Neural Networks~(WNNs)~\cite{aleksander2009brief}, represent a distinct class that overcomes this limitation.
Instead of relying on weighted connections, $\LNN$s employ binary look-up tables~(LUTs) or logic gates to drive neural activity during inference, allowing them to capture highly nonlinear behaviors while avoiding arithmetic operations altogether.
Although state-of-the-art models achieve remarkable performance on small-scale tasks, their scalability remains constrained---whether due to limited expressiveness~\cite{susskind2023uleen,andronic2023polylut}, reliance on gradient approximations~\cite{bacellar2024differentiable}, double-exponential growth in parameter requirements~\cite{petersen2022deep,petersen2024convolutional}, or large discretization gaps~\cite{kim2023deep,yousefi2025mind}, 
describing an accuracy mismatch between the training using relaxations and discrete inference.
We overcome these issues by proposing \textbf{\underline{WA}lsh \underline{R}elaxation for \underline{P}robabilistic~($\WARP$)} $\LNN$s, based on differentiable relaxation, similar to Differentiable Logic Gate Networks~($\DLGN$)~\cite{petersen2022deep}.
We propose representing Boolean functions with the Walsh--Hadamard~(WH) transform~\cite{kunz1979equivalence}, providing a compact and differentiable parameterization of a deep $\LNN$ enabling an exponential reduction in the number of parameters compared to $\DLGN$s.
Publishing our code will bridge the gap created by the lack of publicly available convolutional $\DLGN$ implementations for the community.
Our key contributions, depicted in~\cref{fig:overview}, are summarized as follows:
\begin{itemize}[leftmargin=10pt]
    \item \textbf{$\WARP$ Parametrization:}
    We propose $\WARP$-LNNs, utilizing a new differentiable relaxation of higher-input Boolean functions. This does not only increases expressiveness over $\DLGN$s, but also reduces the number of parameters exponentially.
    Through residual initialization and stochastic smoothing, we enable fast convergence, stable gradients and improve discrete inference accuracy.
    \item \textbf{Representation Generality:} We show that $\WARP$ is the most general and efficient parametrization for representing Boolean functions maintaining full expressiveness. We shed light on the interconnections between state-of-the-art models from the literature and show that they are deduced by basis changes, weight restrictions, approximative gradients or non-invertable input mappings. 
    \item \textbf{Learnable Thresholding:} For $\LNN$s, data has to be binarized in inference. While previous methods use data-based binarization, we propose to a differentiable relaxation to enable learning the thresholds during training.
    \item \textbf{Experimental Effectiveness:} We demonstrate versatility across diverse settings: $\WARP$ reliably learns higher-input LUTs in deep architectures where prior SOTA methods struggle, integrates seamlessly with existing approaches to reduce the number of operations in convolutions, and leverages learnable thresholding to substantially lower the required bit precision at inference time.
    
\end{itemize}

\section{Related Work}
\label{sec:related_work}

TreeLUT~\cite{khataei2025treelut} combines gradient-boosted trees with LUT mappings for efficient inference.
For the reminder of this section, however, we will focus on depth-scalable $\LNN$s.

\paragraph{Function Approximation}

Early work on $\LNN$s explored single-layered perceptrons for approximating LUTs, enabling extremely efficient inference~\cite{susskind2022weightless,susskind2023uleen}.
However, an $n$-input LUT has a known VC-dimension of $2^n$~\cite{carneiro2019exact}, while
a single-layered perceptron with $n$ inputs has a VC-dimension of $n + 1$, limiting expressiveness.
Although extensions to higher-degree polynomial basis~\cite{andronic2023polylut} and multi-layer perceptrons~\cite{andronic2024neuralut,andronic2025neuralut,weng2025greater,hoang2025kanel} have been proposed, these models fall short on expressivity and are faced with exponential training complexity. 

\paragraph{Linear Combination}

To address these limitations, $\DLGN$s utilize differentiable relaxations of logic gates, enabling gradient-based training of multi-layer architectures~\cite{petersen2022deep,petersen2024convolutional}.
Advances in connection learning~\cite{mommen2025method,yue2024learning,kresse2025scalable} have improved flexibility, but training $\DLGN$s introduces a discretization gap between training and discrete inference.
Recent work has sought to bridge this gap by stochastic relaxation~\cite{kim2023deep,yousefi2025mind}.

\paragraph{Direct Parametrization}

$\DLGN$s remain restricted to two-input gates and suffer from double-exponential parameter growth, limiting scalability.
To account for higher-degree LUTs, the authors of~\cite{bacellar2024differentiable} propose Differentiable WNNs~(DWNs) to directly parametrize LUTs, leading to a logarithmic parameter reduction compared to $\DLGN$s.
For obtaining differentiability, the addressing function is relaxed to continuous values and gradients are approximated, degrading performance for deep models.

\paragraph{Basis Parametrization}

Recently, alternative parametrization schemes for representing Boolean functions have been developed~\cite{ruttgers2025light,ramirez2025llnn}.
While overcoming the restrictions of double-exponentially many parameters from $\DLGN$s and approximative gradients of DWNs, these methods still suffer from large discretization errors.
This leads to instability in weight initialization and vanishing gradients. 

\section{Background}
\label{sec:background}

\begin{figure*}[t]
    \centering
    \begin{subfigure}{0.31\textwidth}
        \centering
        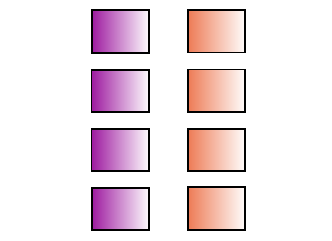
        \caption{$\DLGN$~\cite{petersen2022deep}.}
    \end{subfigure}\hfill
    \begin{subfigure}{0.31\textwidth}
        \centering
        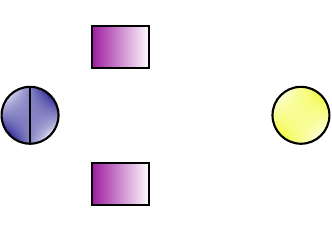
        \caption{$\LLNN$~\cite{ruttgers2025light}.}
    \end{subfigure}\hfill
    \begin{subfigure}{0.31\textwidth}
        \centering
        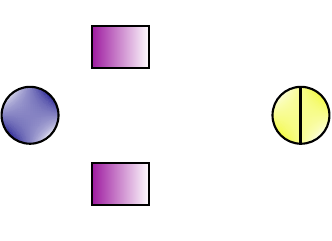
        \caption{$\WARP$ (ours).}
    \end{subfigure}
    \caption{Our parametrization compared to methods from the literature for $n=1$. The $S$-shaped curve indicates the sigmoid function.}
\end{figure*}

Before we propose our reparametrization scheme for Boolean functions in~\cref{sec:meth}, we begin by reviewing fully expressive $\LNN$ approaches from the literature and highlight their limitations.
For the reminder of this paper, we fix the following notation. 
Matrices/vectors are denoted by upper/lower case letters, e.g., $\mat A$/$\vec a$, and we denote the binary representation of integer $i$ with $\vec b^{(i)}\in\mathbb \{0,1\}^n$ starting with the least significant bit, $n=\left\lceil\log_2 i\right\rceil$. 
Further, let $[n]=\{1,\dots,n\}$, $\sigma(\cdot)$ represent the sigmoid function, $\mathds{1}$ the indicator function, and $\veclut{f}\in\{0,1\}^{2^n}$ the LUT of a function $f:\{0,1\}^n\to\{0,1\}$. Finally, let $\btp:\{0,1\}\to \{-1,1\}$ denote the bijection $\btp(x)=1-2x$.

\subsection{Differentiable Logic Gate Networks}
\label{sec:background:dlgn}

Networks composed of logical operations offer a discrete and interpretable framework for computation, but directly optimizing their combinational structure is computationally intractable due to the exponential search space.
To enable efficient gradient-based optimization, prior work introduces Differentiable Logic Gate Networks~($\DLGN$s) by relaxing discrete gate assignments into a continuous, trainable parametrization~\cite{petersen2022deep,petersen2024convolutional}.

\paragraph{Parametrization}

$\DLGN$s find the ideal combination of gates by calculating a weighted sum over all possible gates at each node of the computational graph and calculating a gradient w.r.t. these weights.
The approach involves continuously relaxing onto the simplex spanning all possible Boolean functions:
\begin{align}\label{eq:dlgn_param}
f_{\DLGN}(\vec x)=\sum_{i\in[2^{2^n}]}\alpha_i\basis{\DLGN}{i}(\vec x), \quad \sum_{i\in[2^{2^n}]}\alpha_i=1,
\end{align}
with $\alpha_i\ge 0$ and $\basis{\DLGN}{i}:\mathbb R^n\to[0,1]$ being a probabilistic surrogate of the corresponding Boolean function in form of a multilinear polynomial (see~\cref{tab:binary-gates} for the example $n=2$).
A surrogate formulation is required to accommodate real-valued inputs during training, as the underlying discrete functions are not defined in this domain.
To ensure staying in the probability simplex, a softmax distribution is used for the weights in~\cref{eq:dlgn_param}, i.e., $\alpha_i = \sigma({\alpha'_i}) / \sum_j\sigma({\alpha'_j})$, $\alpha'_i\in\mathbb R$.

\paragraph{Limitations}

Originally developed for two-input logic gates, this approach can be extended to $n$-ary Boolean functions.
However, a fundamental limitation of this approach lies in the double-exponential growth of complexity with increasing neuron arity: a single $6$-input Boolean function, for example, would require $2^{2^6}=1.8447\cdot 10^{19}$ parameters, making the architecture computationally intractable.
Consequently, $\DLGN$s are effectively constrained to learning logic gates with very few inputs, as handling  parameters per gate becomes rapidly infeasible as $n$ increases.

During inference, all softmax operations are replaced by argmax selections, effectively rounding each neuron to the binary gate associated with the largest weight.
The resulting logic gate circuit can be directly embedded in hardware such as FPGAs or ASICs for highly efficient execution.
A key limitation of $\DLGN$s is the misalignment between training and inference. Furthermore, the presence of redundancy introduces ambiguity~\cite{ruttgers2025light}.

\subsection{Differentiable Weightless Neural Networks}
\label{sec:background:dwn}

To overcome the limitation of having double-exponentially many parameters per neuron, the authors of~\cite{bacellar2024differentiable} propose Differentiable WNNs~(DWNs) by
\begin{align}
    f_{\text{DWN}}(\vec x)
    =\sum_{i\in[2^n]}\beta_i\mathds{1}_{\vec b^{(i)}}\left(\btp^{-1}\left(\sgn(\vec x)\right)\right)=\beta_j, \label{eq:dwn}
\end{align}
where $j$ corresponds to the index where the argument of $\btp^{-1}\left(\sgn(\vec x)\right)=\vec b^{(j)}$.
Notably, DWNs have no restrictions on the weights, i.e., $\vec{\beta}\in\mathbb R^{2^n}$, and thus propagate arbitrary real vectors $\vec x\in\mathbb R^n$ through the network.

\paragraph{Limitations}

DWNs have the problem of a non-differentiable parametrization, since the input of the previous layer is discretized to the nearest address for accessing the LUT parameters $\beta_i$.
This issue is approached by using an extended finite difference method which accounts for variations in the addressed position.
However, this is an inaccurate approximation and leads to training instabilities for deep models (see~\cref{sec:experiments}).
For discrete inference, the LUTs are obtained by $\lut{f_{\DWN}}_i=\btp^{-1}\left(\sgn(\beta_i)\right)$.

\subsection{Light Logic Neural Networks}
\label{sec:background:ldlgn}

To overcome non-differentiability, the authors of~\cite{ruttgers2025light} and~\cite{ramirez2025llnn} independently developed similar parametrizations.
While \cite{ramirez2025llnn} denote their approach as LUT $\LNN$s~($\LLNN$s), \cite{ruttgers2025light} uses the term Light $\DLGN$s~($\LIGHT$).
It uses that every $n$-ary Boolean function can be written as
\begin{align}\label{eq:ldlgn_param}
    f_{\LLNN}(\vec x)=\sum_{i\in[2^n]}\gamma_i\basis{\mathds{1}}{i}(\vec x),\quad \gamma_i\in[0,1],
\end{align}
where $\basis{\mathds{1}}{i}$ are indicator polynomials of the form
\begin{align}\label{eq:indicator_base}
    \basis{\mathds{1}}{i}(\vec x)=\prod_{k\in[n]}x_k^{b_k^{(i)}}(1-x_k)^{1-b_k^{(i)}}.
\end{align}
Since $\basis{\mathds{1}}{i}(\vec x)=\mathds{1}_{\vec b^{(i)}}(\vec x)$ for $\vec{x}\in\mathbb \{0,1\}^n$, the Boolean functions are obtained at the corner points of the hypercube $\vec{\gamma}\in\mathbb \{0,1\}^n$.
$\gamma_i\in[0,1]$ is enforced by using the sigmoid function on arbitrary real weights, i.e., $\gamma_i=\sigma(\gamma_i')$, $\gamma_i'\in\mathbb R$.

\paragraph{Limitations}

Even though this reparametrization greatly reduces the number of parameters per neuron, the weight parameters in~\cref{eq:ldlgn_param} are restricted to the interval $[0,1]$.
This leads to undesirable effects in the optimization landscape, resulting in vanishing gradients~\cite{ruttgers2025light}.

After training, $\LLNN$s discretize neurons into LUTs by $\lut{f_{\LLNN}}_i=\btp^{-1}\left(\sgn(\btp(\beta_i))\right).$
Small perturbations in the parameter space---especially in sharp loss landscapes---can cause significant shifts in behavior, as the selected LUT may not reflect the neuron’s actual functional output.
This can lead to a large discretization gap, heavily degrading inference performance after training.

\section{$\WARP$ Logic Neural Networks}
\label{sec:meth}

For the analysis of Boolean functions---such as investigating spectral concentration~\cite{o2014analysis} and learnability~\cite{hellerstein2007pac}---one often considers the symmetrical hypercube $ \{-1,1\}^n$ as domain instead, that is functions of the form $g: \{-1,1\}^n\to \{-1,1\}$.

\subsection{Walsh--Hadamard Transform}
\label{sec:methodology:wh}

A helpful tool for the analysis is the \emph{Fourier expansion}
\begin{align}\label{eq:wh_transform}
    g(\vec s)=\sum_{i\in[2^n]}\theta_i \basis{\text{W}}{i}\left(\vec s\right),\quad \basis{\text{W}}{i}(\vec s)=\prod_{k\in[n]}s_k^{b_k^{(i)}}
    .
\end{align}
The functions $\basis{\text{W}}{i}$ form an orthonormal basis for all functions over $\{-1,1\}^n$.
This is also known as the \emph{Walsh--Hadamard}~(WH) transform and $\theta_i$ can be written as
\begin{align}
    \theta_i=\frac{1}{\sqrt{2^n}}\sum_{\vec s}g(\vec s) \basis{\text{W}}{i}(\vec s)=\frac{1}{\sqrt{2^n}} \sum_{j\in[2^n]}H_{ij}\lut{g}_j,
    \label{eq:walsh_params}
\end{align}
where $\mat H$ denotes the $n$-Hadamard matrix~\cite{hedayat1978hadamard}.
The spectral regularization for DWNs exactly corresponds to $\mathcal L_2$-regularization of the parameters $\vec\theta$.

Intuitively, this expansion uses simple polynomial basis functions---individual variables, pairwise products, and higher-order interactions---to capture the structure of the Boolean function. 
Hence the WH representation provides a compressed and structured parametrization of Boolean functions, where the $2^{2^n}$ Boolean functions are in bijection with a finite subset of the $2^n$-dimensional lattice.

\paragraph{Example}

In the special case of two inputs $\vec s=(u,v)\in\{-1,1\}^2$, every function $g:\{-1,1\}^2\to\{-1,1\}$ admits a decomposition with only four coefficients:
\begin{align*}
g(u,v) = \sgn\left(\theta_1 + \theta_2 u + \theta_3 v + \theta_4 (u \cdot v)\right),
\end{align*}
where \(\theta_1\) encodes the constant bias (tendency toward $+1$ or $-1$), \(\theta_2\) and \(\theta_3\) encode dependence on the individual inputs, and \(\theta_4\) encodes the interaction term between the inputs.
For example, the coefficients $(\theta_1,\theta_2,\theta_3,\theta_4)=(0,0,0,1)$ correspond to the \textsc{XOR} gate, while the \textsc{AND} gate can be expressed as $(\theta_1,\theta_2,\theta_3,\theta_4)=(-\tfrac{1}{2},\tfrac{1}{2},\tfrac{1}{2},\tfrac{1}{2})$ (see~\cref{tab:binary-gates} for the full list of gates and coefficients for $n=2$).

\subsection{Parametrization}
\label{sec:meth:param}

Instead of being restricted to the binary domain $\{-1,1\}^n$, the WH transform in~\cref{eq:wh_transform} decomposes an arbitrary input vector into a superposition of basis functions.
This provides a differentiable relaxation for Boolean functions
\begin{align}\label{eq:wh_param}
    f_{\WARP}(\vec x)=\sigma\left(\frac{1}{\tau}\sum_{i\in[2^n]}\theta_i\basis{\text{W}}{i}(\btp(\vec x))\right).
\end{align}
\begin{theorem}\label{theo:generality}
The $\WARP$ parametrization is the most parameter-efficient representation of any boolean function and previous approaches are special cases that either introduce redundancy ($\DLGN$), errors through approximation ($\DWN$) or restrictions on parameters ($\LLNN$ and $\LIGHT$).
\end{theorem}
The proof can be found in~\cref{sec:app_generality}.
The decomposition in~\cref{eq:wh_param} provides a compact and differentiable parameterization with $2^n$ parameters, while still allowing to collapse back into one of the $2^{2^n}$ exact Boolean functions at inference time by choosing the closest LUT:
\begin{align}
    \lut{f_{\WARP}}_i=\btp^{-1}\left(\sgn\left(\sum\nolimits_j H_{ij}\theta_j\right)\right).
    \label{eq:discretization}
\end{align}
Exploiting the recursive structure of $\mat H$, the WH transform in~\cref{eq:discretization} is efficiently computable in $\mathcal O(n2^n)$~\cite{fino1976unified}.
Further, this discretization is also optimal.
\begin{theorem}\label{theo:discretization}
Under the $\WARP$ parametrization, the discretization that minimizes the error in any $\mathcal L_p$-norm is achieved by applying \eqref{eq:discretization}, or equivalently by choosing the LUT whose Walsh–Hadamard coefficients are closest to $\vec\theta$.
\end{theorem}
The proof can be found in~\cref{sec:app_discretization}.
Note that our parametrization in~\cref{eq:wh_param} has three key differences compared to $\LIGHT$ in~\cref{eq:ldlgn_param}:
\begin{enumerate*}[label=(\roman*)]
\item The removal of weight constraints prevents instability during initialization and can exacerbate vanishing gradient phenomena,
\item nonlinearities are introduced explicitly through the sigmoid activation in contrast to $\LIGHT$, where nonlinear dependence on the input is encoded implicitly by polynomial basis expansions, and
\item a continuous, hyperparameter $\tau$ is incorporated, enabling additional flexibility in shaping the representation.
\end{enumerate*}
Morever, we introduce stochastic smoothing and residual initialization for stabilizing continuous training and improving discrete inference.


\paragraph{Stochastic Smoothing}
\label{sec:meth:gap}

Using the discretization in~\cref{eq:discretization}, we face the same issues as for $\LLNN$.
In sharp loss regions, small parameter changes misalign LUT selections, leading to a large discretization gap and poor inference performance.

For mitigating this effect, we note that the sigmoid function characterizes a Bernoulli distribution, representing the probability of an input vector being discretized to a binary output, i.e., $\mathbb P(\lut{f_{\WARP}}_i=1)=f_{\WARP}(\vec b^{(i)})$.
To introduce stochasticity into the training process, we adopt Gumbel–Sigmoid reparameterization, a binary variant of the Gumbel–Softmax (Concrete) distribution, enabling gradient-based training with discrete variables~\cite{maddison2017concrete,jang2017categorical}.
Differentiable samples $z(\vec x)~\sim\text{Ber}(f_{\WARP}(\vec x))$ are obtained by sampling $G_1,G_2\sim\text{Gumbel}(0,1)$
\begin{align}\label{eq:wh_param_gumbel}
    z(\vec x)=\sigma\left(\frac{1}{\tau}\sum_{i\in[2^n]}\theta_i\basis{\text{W}}{i}(\btp(\vec x))+G_1-G_2\right),
\end{align}
with temperature parameter~$\tau$.
Lower values of $\tau$ yield a tighter relaxation towards the discrete domain~$\{0,1\}$, thereby reducing the discretization gap.
However, choosing $\tau$ too small can lead to numerical instabilities such as exploding gradients.
In our experiments, we evaluate the performance of using discretization already during training in the forward pass as a straight-through estimator (see~\cref{sec:experiments}).


\paragraph{Residual Initialization}
\label{sec:initialization}

To avoid vanishing gradients and achieving a reduced discretization gap, the authors of~\cite{petersen2024convolutional} propose to use \emph{residual initialization}~(RI).
Inspired by residual connections~\cite{he2016deep}, a large initial probability is assigned to the pass-through function of a single input variable, e.g., $f(\vec x)=x_n$.
Similarly to $\DLGN$s, this initialization is straight-forward for $\WARP$ for a chosen probability $p$ to the pass-through function.
\begin{align}
    \theta_{2^{n-1}+1}=\tau\sigma^{-1}(p),\ \theta_{i}=0,\ \forall i\neq2^{n-1}+1.
    \label{eq:initialization}
\end{align}
A proof can be found in~\cref{sec:app_initialization} in the appendix.

\subsection{Learnable Thresholding}
\label{sec:learnable_thresholds}

Since we learn a Boolean function for inference, data has to be binarized for compatibility, usually done by thermometer thresholding.
Given a threshold matrix $\mat \Omega\in\mathbb R^{l\times d}$, a data point $\vec x\in\mathbb R^{d}$ is binarized by checking $\omega_{ij}\le x_j$.
State-of-the-art examples include uniform~\cite{petersen2024convolutional} and distributive~\cite{bacellar2024differentiable} thresholding.

Instead of only including data-specific information, we propose to include information on the given task by learning the matrix $\mat \Omega$ simultaneously with the network parameters.
For thermometer thresholding, the row entries have to be monotonically increasing, i.e., $\Omega_{ij}\le \Omega_{ik}$ for $j\le k$.
We follow the approach of learning the first-order differences $\Delta_{ij}$ between subsequent thresholds, by using softplus for smoothly enforcing positivity ($\mathrm{softplus}(\Delta_{ij})=\Omega_{ij+1}-\Omega_{ij}$)
\begin{align*}
    \Omega_{ij}=\sum_{k\le j}\mathrm{softplus}(\Delta_{ij}),\quad \mathrm{softplus}(x)=\log(1+e^x).
\end{align*}
Differentiable thresholding is obtained by using
$\mathds{1}\lbrace\Omega_{ij}\le x_j\rbrace\approx\sigma\left((x_j-\Omega_{ij})/\rho\right)$,
where $\rho\ge 0$ is a temperature parameter controlling smoothness. Similar to~\cref{eq:wh_param_gumbel}, we can introduce Gumbel noise to stabilize the training process, since hard thresholds are computed during inference.

\begin{figure*}[t!]
    \centering
    \includegraphics[width=\textwidth]{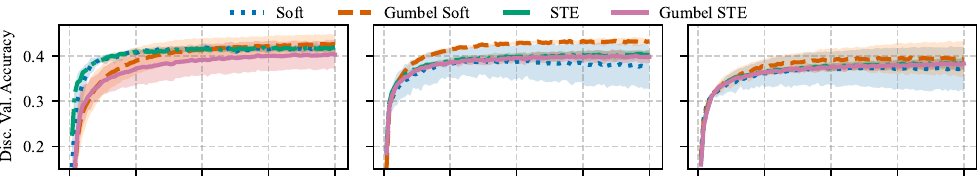}
    \includegraphics[width=\textwidth]{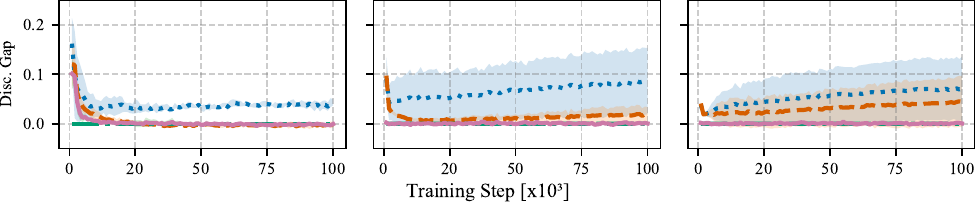}
    \caption{Discrete validation accuracy (top) and discretization gap (bottom) for $\WARP$ on CIFAR-10 comparing different parametrization methods and varying LUT sizes, $n=2$ (left), $n=4$ (middle) and $n=6$ (right).}
    \label{fig:lutrank_gumbel}
\end{figure*}
\section{Experiments}
\label{sec:experiments}

We conduct three sets of experiments. 
First, we highlight the training dynamics of $\WARP$ $\LNN$s and show the improving effect of our proposed residual initialization and stochastic smoothing.
Secondly, we show the effect of learnable thresholding and that it is not restricted to $\WARP$, but is usable for other $\LNN$s benefiting from continuous inputs.
Lastly, the performance of $\WARP$ is assessed in comparison to state-of-the-art methods from the literature.

\subsection{Training Dynamics}
\label{sec:dynamics}

We aim to analyze the training and inference performance of $\WARP$ with using different input-arities for every single neuron.
For this, we evaluate model performance on the CIFAR-10 dataset following standard pre-processing and an 80/20 split in training / validation data.
The architectures of the models used are shown in~\cref{tab:mnist_smallres}.
We adopt the “small” model architecture introduced in the original $\DLGN$ paper~\cite{petersen2022deep} as the baseline and initialize the connections between neurons randomly.
Training is done with residual weight initialization and averaged over 5 different seeds with indicating $95\%$ confidence intervals.

In~\cref{fig:lutrank_gumbel}, we compare LUT sizes of $n=2,4,6$ (from left to right) while varying four different methods for training process. 
Soft refers to the proposed vanilla neuron parametrization in~\cref{eq:wh_param}, while Gumbel Soft introduces stochastic smoothing via~\cref{eq:wh_param_gumbel}.
Further, we also compare corresponding straight-through estimators~(STE) by discretizing the inputs and relaxed parametrizations in the forward pass.
Moreover, to obtain intercomparability between different LUT sizes, we reduce the number of neurons for increasing LUT size proportional to their fraction, that is a reduction of $2\times$ for $n=4$ and $3\times$ for $n=6$.
This leads to the same number of total connections, hinting towards possible deployment advantages for larger LUT sizes on logic-based hardware.
Further hyperparameters can be found in~\cref{sec:app_fmnist}.

The upper plot panel shows the discrete validation accuracy during training, indicating that different LUT sizes show similar performance.
Even though the accuracy is slightly lower for $n=6$, we stress that we do not aim to obtain optimal performance in this experiment.
Rather than fine-tuning the temperature parameter $\tau$, we emphasize that this result already demonstrates that $\WARP$ can be scaled to logic blocks with more than two inputs.

As is evident from the bottom plot panel, stochastic smoothing by injecting Gumbel noise reduces the discretization gap, defined as the difference of the relaxed validation accuracy and the discrete validation accuracy.
While it is constantly 0 for the STEs, Gumbel Soft is able to reduce the gap by smoothing the optimization landscape.
Even though Soft has the best relaxed validation accuracy performance, the discretization gap starts to increase with more training steps, especially for larger LUT sizes.
Similar effects are observed for different datasets, architectures and hyperparameter configurations, as is evident from~\cref{sec:app_fmnist}.

\subsection{Learnable Thresholding}
\label{sec:thresholding}

\begin{figure}[t!]
\includegraphics[width=\linewidth]{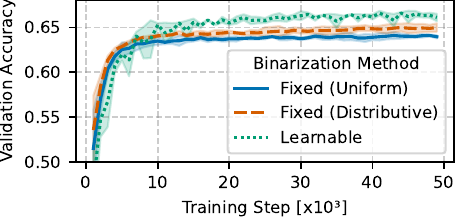}
\caption{Validation accuracy on CIFAR-10 for different binarization methods. The model architecture resembles LogicTreeNet-M, and the $\DLGN$ parametrization according to~\cref{eq:dlgn_param}. Each learning curve consists of three runs with different random seeds.}
\label{fig:cifar-binarization}
\end{figure}

Next, we investigate the effect of learnable thresholding. This is particularly relevant for datasets with heterogeneous feature distributions, where fixed binarization schemes may be suboptimal. To this end, we evaluate different thresholding strategies on the JSC dataset, which exhibits features with varying value ranges and densities.

We adopt the \emph{sm} architecture proposed in~\cite{bacellar2024differentiable}, consisting of two layers with $6$-input LUTs, with learnable connections between the input features and the first layer. We compare three different thresholding schemes:
(i)~fixed uniform thresholds, (ii)~fixed distributive thresholds, and
(iii)~learnable thresholds, which are optimized jointly with the model parameters during training.

The fixed uniform thresholds are placed at equidistant intervals between the minimum and maximum value of each feature. In contrast, distributive thresholding follows the approach of DWN~\cite{bacellar2024differentiable}, where thresholds are chosen according to the empirical feature distributions. Learnable thresholding initializes thresholds in the same way, but allows them to adapt during training via gradient-based optimization.

\cref{fig:jsc-binarization} shows the best validation accuracy on the JSC dataset over $10$ training runs with different random seeds, plotted as a function of the number of bits per feature after binarization. We observe that distributive thresholding already provides a substantial performance improvement over uniform thresholding, which also exhibits the highest variance across runs. Learnable thresholding further improves performance, consistently outperforming distributive thresholding in the low bit-width regime. While distributive thresholding requires approximately $20$ bits per feature to reach peak performance, learnable thresholding achieves a comparable accuracy of roughly $70\%$ with only $5$ bits. Note that~\cite{bacellar2024differentiable} deploys a $200$ bit binarization scheme.

We found the impact of learnable thresholding in the absence of learnable input connections less pronounced, but still significant.
Although learnable thresholding is best suited for datasets with strongly varying feature statistics, we additionally evaluate its effect on image data from CIFAR-10. For this experiment, we employ the LogicTreeNet-M architecture and the corresponding parametrization from $\DLGN$~\cite{petersen2024convolutional}. In the original work, each pixel is encoded as two bits per color channel, using the fixed uniform thresholding. \cref{fig:cifar-binarization} depicts the average training curves over three runs for all three thresholding strategies. Once again, learnable thresholding yields a significant improvement in final accuracy, albeit with slightly reduced training stability compared to fixed schemes.

\begin{figure}[t!]
\includegraphics[width=\linewidth]{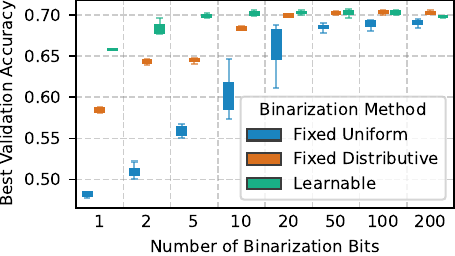}
\caption{Validation accuracy on JSC for different binarization methods vs the number of bits for the binarization of each feature. The models' architectures resemble DWN(n=6, \textit{sm}) with learnable connections in the first layer. Each configuration was run with ten different random seeds.}
\label{fig:jsc-binarization}
\end{figure}

In this setting, the two thresholds are learned globally, i.e., shared across all pixels and channels. We also experimented with channel-wise and pixel-wise learnable thresholds but found the increased hyperparameter sensitivity and reduced training stability to outweigh the potential performance gains.
A complete list of hyperparameters and training settings is provided in~\cref{tab:hyperparameters}.

\subsection{Higher-Rank LUTs in Convolutional Kernels}
\label{sec:highrank_conv}

\begin{figure}[t!]
\includegraphics[width=\linewidth]{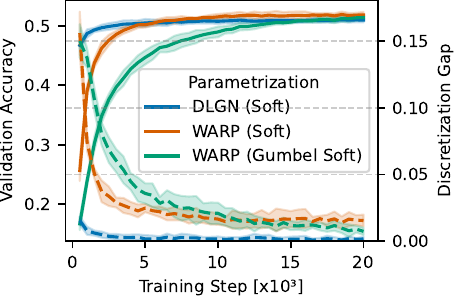}
\caption{Validation accuracy on CIFAR-10 for a small model with a single convolutional layer. The convolution kernels are implemented as binary trees of depth four and four-input trees of depth two for $\DLGN$, and $\WARP$, respectively.}
\label{fig:cifar10-mixed}
\end{figure}
\begin{figure*}[t!]
    \centering
    \includegraphics[width=\textwidth]{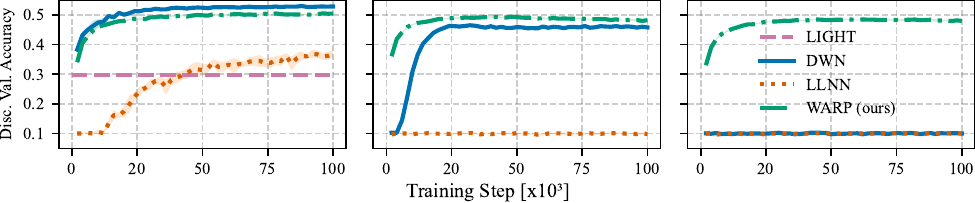}
    \caption{Discrete validation accuracy (top) for $\DWN$~\cite{bacellar2024differentiable}, $\LLNN$~\cite{ramirez2025llnn}, $\LIGHT$~\cite{ruttgers2025light} and $\WARP$ (ours) on CIFAR-10 with a deep architecture comparing varying LUT sizes, $n=2$ (left), $n=4$ (middle) and $n=6$ (right).}
    \label{fig:deep}
\end{figure*}
We investigate the effect of employing higher-rank LUTs within convolutional kernels using $\WARP$.
Recall that in $\DLGN$, convolutional kernels with $2^d$ inputs are implemented as binary trees of depth $d$, which cannot represent all Boolean functions over its $2^d$ inputs.

Using $\WARP$, we instead construct convolutional kernels from higher-input LUTs. Specifically, we implement kernels composed of $4$-input LUTs arranged in a tree of depth $d=2$, resulting in an effective receptive field of $4^2$ inputs. We compare this configuration to a $\DLGN$ kernel with a binary tree depth of $d=4$, which aggregates $2^4$ inputs per kernel.

Despite having the same number of inputs, the two approaches differ in parameter count: A $\DLGN$ kernel with depth $4$ requires $16 \cdot$(8+4+2+1) trainable parameters, whereas the corresponding $\WARP$ kernel requires only $16 \cdot$(4+1) parameters. Moreover, the evaluation of trees is fundamentally sequential and cannot be parallelized, hence the shallower $\WARP$ tree reduces the number of required sequential operations.
As a result, $\WARP$ kernels provide higher expressive power with fewer parameters and improved parallelism.

\cref{fig:cifar10-mixed} compares these two kernel implementations on CIFAR-10. We employ a small model consisting of a single convolutional layer with a skip connection, which combines its output with the previous layer using a relaxed OR operation. All remaining dense connections in the network are implemented using logic gates, such that only the convolutional kernel construction differs between the two models. The architecture is described in more detail in \cref{tab:clgn_cifar10_smallres}.

The results show that $\WARP$ kernels with higher-input LUTs achieve slightly higher final discrete accuracy than the $\DLGN$ baseline. However, this improvement comes at the cost of slower convergence during training, indicating a trade-off between expressivity and optimization difficulty when increasing LUT-arity in convolutional settings. We attribute the slower convergence to the increased discretization gap, which currently prevents this approach from efficiently scaling to deeper architectures such those in \cref{fig:cifar-binarization}.

\subsection{Comparison to $n$-LUT State-of-the-Art}
\label{sec:sota}

We show the advantages of $\WARP$ over other fully expressive $\LNN$s from the literature.
For comparison, we analyze the scalability of $\DWN$s~\cite{bacellar2024differentiable}, $\LLNN$s~\cite{ramirez2025llnn} and $\LIGHT$~\cite{ruttgers2025light} by considering a deep version of the medium-sized model architecture from~\cite{petersen2022deep} for CIFAR-10.
Following~\cite{ruttgers2025light}, we introduce a depth factor $D$ that scales the number of layers in the network.

In~\cref{fig:deep}, we plot the learning curves with respect to the discrete validation accuracy for $n$=2,4,6 (left to right) and $D=3$.
We observe that $\DWN$ still performs well for $n=2$ but its performance strongly degrades for higher LUT ranks. 
This is due to the usage of approximate gradients hindering scalability to larger architectures.
Unfortunately, the code for $\LIGHT$ is not available.
Since results are only reported for $n=2$, we fall back on reporting approximate numbers stated in the paper.
We observe that $\WARP$ strongly outperforms $\LIGHT$, with reaching a higher accuracy.
Comparing to $\LLNN$, we can see that $\WARP$ exhibits better convergence due to the proposed residual initialization for $n=2$, while $\LLNN$ completely fails to learn for $n$=4,6.



\section{Conclusion}



We introduced $\WARP$, a differentiable and parameter-efficient framework for learning logic neural networks based on the Walsh–Hadamard transform.
$\WARP$ provides a fully expressive parametrization of Boolean functions while avoiding the redundancy, approximation errors, and scalability limitations of prior approaches. Through learnable thresholding, residual initialization, and stochastic smoothing, $\WARP$ enables stable training and significantly faster convergence, while narrowing the gap between relaxed training and discrete logic inference.
Our experiments demonstrate that $\WARP$ scales effectively to deeper architectures and higher-arity logic blocks, making it a promising foundation for efficient, hardware-native inference.

While we do not report results on hardware deployment, we expect performance at least on par with existing methods, since WARP learns equivalent logic structures.
In future work, we plan to evaluate WARP on hardware targets to quantify inference latency and resource utilization, and to benchmark performance on larger-scale datasets.



\section*{Impact Statement}

This paper presents work whose goal is to advance the field of Machine Learning.
There are many potential societal consequences of our work, none which we feel must be specifically highlighted here.

\section*{Acknowledgments}
This work was supported by the National Science Foundation under Cooperative Agreement PHY-2323298.


\bibliography{mybib}
\bibliographystyle{icml2026}
\newpage
\appendix
\onecolumn

\begin{table}[t!]
\caption{Densely conncted architectures for JSC, MNIST, FashionMNIST and CIFAR-10 datasets.}
  \centering
  \small
  \label{tab:mnist_smallres}
  \begin{tabular}{l c c c c c}
    \toprule
    Dataset & Model & Layers & Neurons/layer & \#Parameters & GroupSum temp \\
    \midrule
    JSC & \textit{sm} & 1 & 50 & 3200 & 1/0.3 \\
    \midrule
    MNIST & small & 6 & 6000 & 48000 & 1/0.1 \\
    \midrule
    FashionMNIST & small & 6 & 6000 & 48000 & 1/0.1 \\
    \midrule
    CIFAR-10 & small & 4 & 12000 & 48000 & 1/0.01 \\
     & medium & 4 & 128000 & 512000 & 1/0.01 \\
    \bottomrule
  \end{tabular}
 \label{tab:model-architecture}
\end{table}

\begin{table}[t!]
\caption{Architecture of the \textbf{CLGN CIFAR-10 Res} model ($k_\text{num}=32, n_\text{bits}=3$).}
  \centering
  \small
  \label{tab:clgn_cifar10_smallres}
  \begin{tabular}{l c c c}
    \toprule
    Layer & Input dimension & Output dimension & Description \\
    \midrule
    ResidualLogicBlock (DLGN) &
    $3n_\text{bits} \times 32 \times 32$ &
    $2k_\text{num} \times 16 \times 16$ &
    Depth $=4$, receptive field $3\times3$, padding 1 \\
    ResidualLogicBlock (WARP) &
    $3n_\text{bits} \times 32 \times 32$ &
    $2k_\text{num} \times 16 \times 16$ &
    Depth $=2$, receptive field $3\times3$, padding 1 \\
    Flatten &
    $2k_\text{num} \times 16 \times 16$ &
    $256 \times 2k_\text{num}$ &
    Spatial flattening of logic feature maps \\
    LogicDense$_1$ &
    $256 \times 2k_\text{num}$ &
    $512k_\text{num}$ &
    Fully connected logic layer \\
    LogicDense$_2$ &
    $512k_\text{num}$ &
    $256k_\text{num}$ &
    Fully connected logic layer \\
    LogicDense$_3$ &
    $256k_\text{num}$ &
    $320k_\text{num}$ &
    Fully connected logic layer \\
    GroupSum &
    $320k_\text{num}$ &
    $10$ &
    Grouped logic aggregation with temperature $\tau$ \\
    \bottomrule
  \end{tabular}
 \label{tab:model-architecture}
\end{table}

\begin{table*}[t!]
\small
\caption{The $16$ Boolean functions of two inputs, sorted lexicographically by their truth tables (ordered as 00, 01, 10, 11), with probabilistic surrogates, WH coefficients $(\gamma_1,\gamma_2,\gamma_3,\gamma_4)$ and indicator coefficients $(\beta_1,\beta_2,\beta_3,\beta_4)$.}
\centering
\renewcommand{\arraystretch}{1.2}
\begin{tabularx}{0.95\textwidth}{c|c|c|c c c c|c|c|c|c}
\toprule
\# & Gate & Formula & 00 & 01 & 10 & 11 & Surrogate & WH & Indicator & Icon\\
\midrule
1  & CONST0   & $0$                 & 0 & 0 & 0 & 0 & $0$ & $(1,\,0,\,0,\,0)$ & $(0,\,0,\,0,\,0)$ &
\adjustbox{valign=c}{\includegraphics[width=1.0cm]{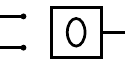}}
\\
2  & AND      & $a \land b$         & 0 & 0 & 0 & 1 & $ab$ & $\left(\tfrac{1}{2},\,\tfrac{1}{2},\,\tfrac{1}{2},\,-\tfrac{1}{2}\right)$ & $(0,\,0,\,0,\,1)$&\adjustbox{valign=c}{\includegraphics[width=1.0cm]{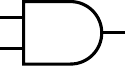}}\\
3  & $a \land \lnot b$ & $a \land \lnot b$   & 0 & 0 & 1 & 0 & $a(1 - b)$ & $\left(\tfrac{1}{2},\,-\tfrac{1}{2},\,\tfrac{1}{2},\,\tfrac{1}{2}\right)$ & $(0,\,0,\,1,\,0)$&\adjustbox{valign=c}{\includegraphics[width=1.0cm]{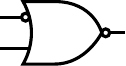}}\\
4 & ID$(a)$    & $a$                 & 0 & 0 & 1 & 1 & $a$ & $(0,\,0,\,1,\,0)$ &$(0,\,0,\,1,\,1)$&\adjustbox{valign=c}{\includegraphics[width=1.0cm]{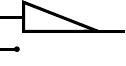}}\\
5  & $\lnot a \land b$ & $\lnot a \land b$   & 0 & 1 & 0 & 0 & $(1 - a)b$ & $\left(\tfrac{1}{2},\,\tfrac{1}{2},\,-\tfrac{1}{2},\,\tfrac{1}{2}\right)$ &$(0,\,1,\,0,\,0)$&\adjustbox{valign=c}{\includegraphics[width=1.0cm]{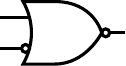}}\\
6 & ID$(b)$    & $b$                 & 0 & 1 & 0 & 1 & $b$ & $(0,\,1,\,0,\,0)$ &$(0,\,1,\,0,\,1)$&\adjustbox{valign=c}{\includegraphics[width=1.0cm]{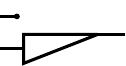}}\\
7  & XOR      & $a \oplus b$        & 0 & 1 & 1 & 0 & $a + b - 2ab$ & $\left(0,\,0,\,0,\,1\right)$ &$(0,\,1,\,1,\,0)$&\adjustbox{valign=c}{\includegraphics[width=1.0cm]{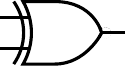}}\\
8  & OR       & $a \lor b$          & 0 & 1 & 1 & 1 & $a + b - ab$ & $\left(-\tfrac{1}{2},\,\tfrac{1}{2},\,\tfrac{1}{2},\,\tfrac{1}{2}\right)$ &$(0,\,1,\,1,\,1)$&\adjustbox{valign=c}{\includegraphics[width=1.0cm]{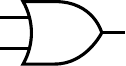}}\\
9  & NOR      & $\lnot(a \lor b)$   & 1 & 0 & 0 & 0 & $1 - a - b + ab$ & $\left(\tfrac{1}{2},\,-\tfrac{1}{2},\,-\tfrac{1}{2},\,-\tfrac{1}{2}\right)$ &$(1,\,0,\,0,\,0)$&\adjustbox{valign=c}{\includegraphics[width=1.0cm]{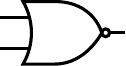}}\\
10  & XNOR     & $\lnot(a \oplus b)$ & 1 & 0 & 0 & 1 & $1 - a - b + 2ab$ & $\left(0,\,0,\,0,\,-1\right)$ &$(1,\,0,\,0,\,1)$&\adjustbox{valign=c}{\includegraphics[width=1.0cm]{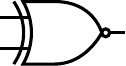}}\\
11 & NOT$(b)$   & $\lnot b$           & 1 & 0 & 1 & 0 & $1 - b$ & $(0,\,-1,\,0,\,0)$& $(1,\,0,\,1,\,0)$&\adjustbox{valign=c}{\includegraphics[width=1.0cm]{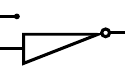}}\\
12 & IMP $(b \to a)$ & $\lnot b \lor a$ & 1 & 0 & 1 & 1 & $1 - b + ab$ & $\left(-\tfrac{1}{2},\,-\tfrac{1}{2},\,\tfrac{1}{2},\,-\tfrac{1}{2}\right)$ & $(1,\,0,\,1,\,1)$&\adjustbox{valign=c}{\includegraphics[width=1.0cm]{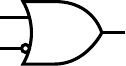}}\\
13 & NOT$(a)$   & $\lnot a$           & 1 & 1 & 0 & 0 & $1 - a$ & $(0,\,0,\,-1,\,0)$ &$(1,\,1,\,0,\,0)$&\adjustbox{valign=c}{\includegraphics[width=1.0cm]{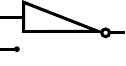}}\\
14 & IMP $(a \to b)$ & $\lnot a \lor b$ & 1 & 1 & 0 & 1 & $1 - a + ab$ & $\left(-\tfrac{1}{2},\,\tfrac{1}{2},\,-\tfrac{1}{2},\,-\tfrac{1}{2}\right)$ &$(1,\,1,\,0,\,1)$&\adjustbox{valign=c}{\includegraphics[width=1.0cm]{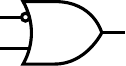}}\\
15  & NAND     & $\lnot(a \land b)$  & 1 & 1 & 1 & 0 & $1 - ab$ & $\left(-\tfrac{1}{2},\,-\tfrac{1}{2},\,-\tfrac{1}{2},\,\tfrac{1}{2}\right)$ &$(1,\,1,\,1,\,0)$&\adjustbox{valign=c}{\includegraphics[width=1.0cm]{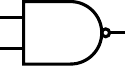}}\\
16  & CONST1   & $1$                 & 1 & 1 & 1 & 1 & $1$ & $(-1,\,0,\,0,\,0)$ &$(1,\,1,\,1,\,1)$&\adjustbox{valign=c}{\includegraphics[width=1.0cm]{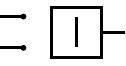}}\\
\bottomrule
\end{tabularx}
\label{tab:binary-gates}
\end{table*}

\section{Generality of $\WARP$}
\label{sec:app_generality}

\begin{proof}
We proof~\cref{theo:generality}.

\paragraph{$\WARP$ is Most Parameter Efficient}

First, we remark that $\basis{\text{W}}{i}$ is an orthonormal basis for all functions $g:\{-1,1\}^n\to\{-1,1\}$
\begin{align*}
    g(\vec s)=\sum_{i\in[2^n]}\theta_i\basis{\text{W}}{i}(\vec s),\quad \left<\basis{\text{W}}{i},\basis{\text{W}}{j}\right>
    =&\sum_{\vec s\in \{-1,1\}^n}\basis{\text{W}}{i}(\vec s)\basis{\text{W}}{j}(\vec s)
    =\sum_{k\in[2^n]}H_{ik}H_{kj}=\delta_{ij}.
\end{align*}
Through orthogonality, it follows that the rows of $\mat H$ are linearly independent which leads to the orthonormal basis argument.
Thus, no other parametrization exists that can represent every $g$ with less parameters.

\paragraph{$\WARP$ is More General Than $\LLNN$}

A basis change w.r.t. $\basis{\mathds 1}{i}$ can be computed by finding a basis transformation matrix $\mat T$ with
\begin{align*}
    T_{ij}=\left<\basis{\mathds 1}{i},\basis{\text{W}}{j}\right>
    =&\sum_{\vec s\in\{-1,1\}^n}\mathds{1}_{\vec b^{(i)}}(\btp^{-1}(\vec s))\basis{\text{W}}{j}(\vec s)
    =\sum_{k\in[2^n]}\delta_{ik}H_{kj}
    =H_{ij}.
\end{align*}
Thus, any $\vec s\in\{-1,1\}^n$ can be transformed between the bases by using the Hadamard matrix $\sum_{ij} H_{ij}\basis{\text{W}}{j}(\vec s)=\basis{\mathds 1}{i}(\vec s)$.
While the $\LLNN$ parametrization restricts their parameters to be in $[0,1]$ through applying the sigmoid function, $\WARP$ allows them to take arbitrary values in $\mathbb R$.
\end{proof}

\paragraph{$\WARP$ is More General Than $\DWN$}

In fact, the $\DWN$ parametrization is similar to $\LLNN$, but uses a straight-through estimator in the forward pass
\begin{align*}
    f_{\DWN}(\vec x)=\btp\left(f_{\LLNN}\left(\btp^{-1}(\sgn(\vec x))\right)\right).
\end{align*}
Hence, it relies on approximate gradients and is less general than $\LLNN$ and through the previous statement, $\DWN$ is also less general than $\WARP$ by means of a non-bijective mapping $\sgn$.

\paragraph{$\WARP$ is Less Redundant and More Parameter-Efficient Than $\DLGN$}

Even though the the functions $\basis{\DLGN}{j}(\vec s)$ are spanning all Boolean functions, they do not form a basis.
$\DLGN$s are parametrized by way more parameters (${2^n}\ll 2^{2^n}$) and a linear projection to the $\LLNN$ basis is given by
\begin{align*}
    \mat P\in \{0,1\}^{2^n\times 2^{2^n}},\ P_{ij}=b^{(j)}_i\in\{0,1\}
    \ \Rightarrow \ \sum_{k\in[2^{2^n}]}P_{ij}\basis{\DLGN}{j}(\vec s)=\basis{\mathds 1}{i}(\vec s),\ \forall i\in[2^n].
\end{align*}
Hence, $\DLGN$s introduce a lot of redundancy.

\section{Parametrization Discretization}
\label{sec:app_discretization}

\begin{proof}
    We proof~\cref{theo:discretization}.
    The discretization error in the $p$-norm is given as
    \begin{align*}
        \left\|\vec f_{\WARP} -\vec t\right\|_p=\left(\sum_i\left\vert f(\vec b^{(i)})-t_i\right\vert^p\right)^{\frac{1}{p}}.
    \end{align*}
    Finding an optimal discretization strategy $\vec t\in\{0,1\}^{2^n}$ leads to
    \begin{align*}
         \argmin_{\vec t\in\{0,1\}^{2^n}}\left\|\vec f_{\WARP} -\vec t\right\|_p
         \ \Leftrightarrow \ \argmin_{t_i\in\{0,1\}}\left\vert f(\vec b^{(i)})-t_i\right\vert,
    \end{align*}
    for all $i\in[2^n]$ and with
    \begin{align*}
    \lut{f_{\WARP}}_i=\argmin_{t\in\{0,1\}}\left\vert f(\vec b^{(i)})-t\right\vert 
    =\argmin_{t\in\{0,1\}}\left\vert \sigma\left(\sum\nolimits_j H_{ij}\theta_j\right)-t\right\vert =\left(1+\sgn\left(\sum\nolimits_jH_{ij}\theta_j\right)\right)/2
    \end{align*}
    we obtain~\cref{eq:discretization}.
    For a vector $\vec s\in\{-1,1\}^{2^n}$, the WH parameters are obtained by $(1/\sqrt{2^n})\mat H\vec s$.
    Looking at the parameter vector $\vec\theta\in\mathbb R^n$
    \begin{align*}
    \argmin_{\vec t\in\{-1,1\}^n}\left\| \vec\theta-\frac{1}{\sqrt{2^n}}\mat H\vec t\right\|_2 
    &=\argmin_{\vec t\in\{-1,1\}^n}\left\|\mat H \vec\theta-\frac{1}{\sqrt{2^n}}\vec t\right\|_2 
    \Leftrightarrow
    \argmin_{t\in\{-1,1\}}\left|\sum\nolimits_jH_{ij}\theta_j-\frac{1}{\sqrt{2^n}}t\right|=\sgn\left(\sum\nolimits_jH_{ij}\theta_j\right)\ \forall i\in[2^n],
\end{align*}
where we used the norm-preserving property of the unitary matrix $\mat H$.
\end{proof}

\section{Parameter Initialization}
\label{sec:app_initialization}

\begin{proof}
We proof the validity of the residual intialization in~\cref{eq:initialization}.
The LUT of the pass-through function $f(\vec x)=x_n$ is given by $\lut{f}_i=0$ if $i<\ell$ and $\lut{f}_i=1$ else with $\ell=2^{n-1}+1$.
Thus, aiming for the $\WARP$ parametrization to represent $f$ with probability $p$ for every LUT entry leads to
\begin{align*}
    \mathbb P(\lut{f_{\WARP}}_i=\lut{f}_i)=p,\ \forall i
    &\Leftrightarrow f_{\WARP}(\vec b^{(i)})=1-p,\ i<\ell\wedge  f_{\WARP}(\vec b^{(i)})=p, \ i\ge \ell \\
    &\Leftrightarrow \sigma\left(\frac{\sqrt{2^n}}{\tau}\sum_jH_{ij}\theta_j\right)=1-p,\ i<\ell\wedge \sigma\left(\frac{\sqrt{2^n}}{\tau}\sum_jH_{ij}\theta_j\right)=p, \ i\ge \ell \\
    &\Leftrightarrow \sum_jH_{ij}\theta_j=-\frac{\tau}{\sqrt{2^n}} \sigma^{-1}(p),\ i<\ell\wedge \sum_jH_{ij}\theta_j=\frac{\tau}{\sqrt{2^n}} \sigma^{-1}(p), \ i\ge \ell,
\end{align*}
with $\sigma^{-1}(1-p)=-\sigma^{-1}(p)$.
Due to the recursive structure of the Hadamard matrix
\begin{align*}
    \mat H^{(n+1)}=
    \frac{1}{\sqrt{2}}
    \begin{pmatrix}
        \mat H^{(n)} & \mat H^{(n)} \\
        \mat H^{(n)} & -\mat H^{(n)}
    \end{pmatrix}, \quad
    \mat H^{(1)}=\frac{1}{\sqrt 2}
    \begin{pmatrix}
        1 & 1 \\
        1 & -1
    \end{pmatrix},
\end{align*}
combining the $2^n$ different linear equations leads to 
\begin{align*}
    \sum_{i<\ell}\sum_jH_{ij}\theta_j&=\frac{2^{n-1}}{\sqrt{2^n}}\left(\theta_1+\theta_{\ell}\right)=-\frac{\tau}{\sqrt{2^n}}\sigma^{-1}(p)2^{n-1}, \\
    \sum_{i\ge\ell}\sum_jH_{ij}\theta_j&=\frac{2^{n-1}}{\sqrt{2^n}}\left(\theta_1-\theta_{\ell}\right)=\frac{\tau}{\sqrt{2^n}} \sigma^{-1}(p)2^{n-1}.
\end{align*}
Subtracting these two equations, we obtain $\theta_{\ell} =\tau \sigma^{-1}(1-p)$ and $\theta_1=0$.
Similarly, it follows that $\theta_i=0$, $\forall i\in [2^n]\setminus\{\ell\}$.

\end{proof}

\paragraph{Gradient Stability}

The entries of the gradient w.r.t. to the input is computed as
\begin{align*}
    \frac{\partial f_{\WARP}(\vec x)}{\partial x_i}=\frac{1}{\tau}\sigma'\left(\frac{1}{\tau}\sum_{i\in[2^n]}\theta_i\basis{\text{W}}{i}(\btp(\vec x))\right)
    \cdot \left(\sum_{i\in[2^n]}\theta_i\frac{\partial}{\partial x_i}\basis{\text{W}}{i}(\btp(\vec x))\right).
\end{align*}
The gradient vanishes if one of the two factors is close to $0$.
For the first factor, we get
\begin{align*}
    \sigma'\left(\frac{1}{\tau}\sum_{i\in[2^n]}\theta_i\basis{\text{W}}{i}(\btp(\vec x))\right)
    =\sigma\left(\frac{1}{\tau}\sum_{i\in[2^n]}\theta_i\basis{\text{W}}{i}(\btp(\vec x))\right)
    \cdot \sigma\left(1-\frac{1}{\tau}\sum_{i\in[2^n]}\theta_i\basis{\text{W}}{i}(\btp(\vec x))\right),
\end{align*}
which is close to $0$ if the argument is $\ll 0$ or $\gg 0$.
For our RI, we have
\begin{align*}
    \sum_{i\in[2^n]}\theta_i\basis{\text{W}}{i}(\btp(\vec x))=\theta_{\ell}\basis{\text{W}}{\ell}(\btp(\vec x))
    =\theta_{\ell}\prod_{k\in[n]}\btp(x_k)^{b_k^{(\ell)}}
    =\theta_{\ell}\btp(x_{\ell}),
\end{align*}
which can be well controlled if $\theta_l$ is not $\ll 0$ or $\gg 0$.

The second factor evaluates to 
\begin{align*}
    \sum_{i\in[2^n]}\theta_i\frac{\partial}{\partial x_i}\basis{\text{W}}{i}(\btp(\vec x))=\theta_{\ell}\frac{\partial}{\partial x_{\ell}}\btp(x_{\ell})=-2\theta_{\ell},
\end{align*}
which is also well controlled if $\theta_l$ is $< 0$ or $> 0$.
Our proposed RI intialization thus prevents gradient vanishing, compared to random initialization~\cite{ruttgers2025light}.
\paragraph{Intercomparability to $\DLGN$}

$\DLGN$s use RI, introducing skip connections by increasing probability of sampling identity gates.
Introducing a parameter $c\ge 0$ controlling this increase, we obtain
\begin{align*}
    \mathbb P(\lut{f_{\DLGN}}_i=1)=f_{\DLGN}(\vec b^{(i)})
    &=\sum_j\alpha_j\basis{\DLGN}{j}(\vec b^{(i)}) \\
    &=\frac{1}{e^c+2^{2^n}-1}\sum_j\lut{j}_ic_j',\quad 
    c_j'=\begin{cases}
        e^c, & \text{if }j=2^{2^{n-1}}-1, \\
        1, & \text{else.}
        \quad 
    \end{cases}\\
    &=\frac{1}{e^c+2^{2^n}-1}\left(\lut{j'}_ie^c+\sum_{j\neq j'}\lut{j}_i\right) \\
    &=\frac{1}{e^c+2^{2^n}-1}\left(\delta_{i>2^{n-1}}e^c+2^{2^n-1}-\delta_{i>2^{n-1}}\right).
\end{align*}
To obain intercomparability between methods, we want $\mathbb P(\lut{f_{\DLGN}}_i=\lut{f}_i)=\mathbb P(\lut{f_{\WARP}}_i=\lut{f}_i)=p$.
Thus,
\begin{align*}
    \theta_{\ell}=\tau \sigma^{-1}(1-p)=(2^n-1)\log(2)-\log\left(e^c+2^{2^n-1}-1\right).
\end{align*}
For obtaining the corresponding value of $c$ representing a certain probability $p$, we compute 
\begin{align*}
    \frac{1}{e^c+2^{2^n}-1}\left(2^{2^n-1}\right)
    &=1-p \\
    \Leftrightarrow e^c+2^{2^n}-1&=\frac{2^{2^n-1}}{1-p} \\
    \Leftrightarrow c&=\left(\frac{2^{2^n-1}}{1-p}-2^{2^n}+1\right).
\end{align*}

\section{Training Dynamics on MNIST and Fashion-MNIST}
\label{sec:app_fmnist}

We run a similar experiment to the one conducted in~\cref{sec:dynamics}, but on different datasets.
That is, we plot the discrete validation accuracy and the discretization gap for LUT sizes $n=2,4,6$ on MNIST and FashionMNIST, varying the parametrization with using Gumbel noise injection and STEs.
For both datasets, we utilize the small model from~\cite{petersen2022deep} and the plots are given in~\cref{fig:lutrank_gumbel_fmnist,fig:lutrank_gumbel_mnist}.
\begin{figure*}
    \centering
    \includegraphics[width=\textwidth]{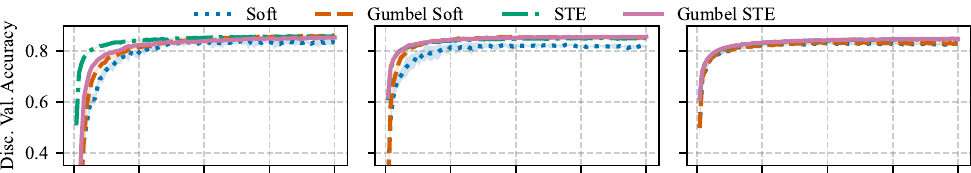}
    \includegraphics[width=\textwidth]{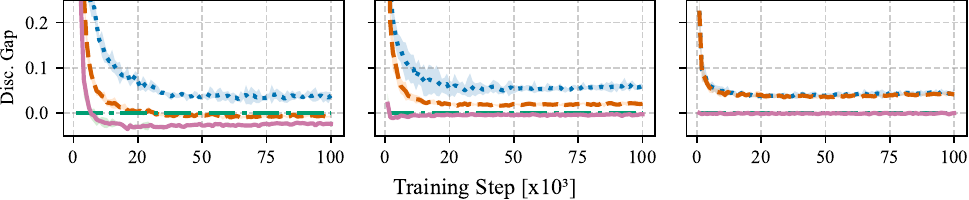}
    \caption{Discrete validation accuracy (top) and discretization gap (bottom) for $\WARP$ on FashionMNIST comparing different parametrization methods and varying LUT sizes, $n=2$ (left), $n=4$ (middle) and $n=6$ (right).}
    \label{fig:lutrank_gumbel_fmnist}
\end{figure*}
\begin{figure*}
    \centering
    \includegraphics[width=\textwidth]{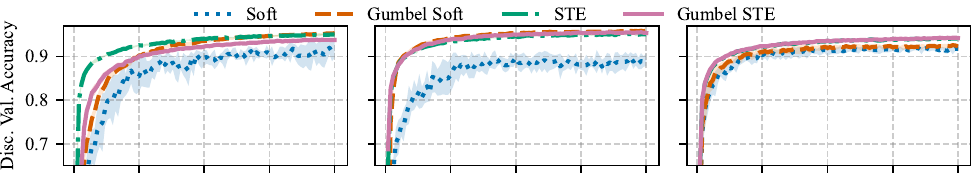}
    \includegraphics[width=\textwidth]{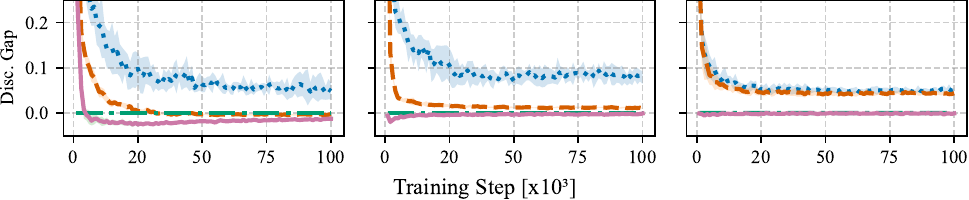}
    \caption{Discrete validation accuracy (top) and discretization gap (bottom) for $\WARP$ on MNIST comparing different parametrization methods and varying LUT sizes, $n=2$ (left), $n=4$ (middle) and $n=6$ (right).}
    \label{fig:lutrank_gumbel_mnist}
\end{figure*}
We scale the temperature parameter $\tau$ proportional to the number of parameters to accomodate for the number of summands, that is $\tau=1$ for $n=2$, $\tau=4$ for $n=4$, and $\tau=16$ for $n=6$.

\begin{table}[t!]
\centering
\caption{Hyperparameters used in the thresholding and high-input convolution experiments. Parameters not listed were adopted from the original literature of the respective model architectures.}
\label{tab:hyperparameters}
\begin{tabular}{l l l}
\toprule
\textbf{Experiment} & \textbf{Hyperparameter} & \textbf{Value} \\
\midrule
\multirow{8}{*}{\textbf{Learnable Thresholding (JSC)}} 
 & Batch size & 128 \\
 & Forward sampling & Gumbel-Soft \\
 & Learning rate & 0.002 \\
 & WARP sigmoid temperature & 1.0 \\
 & Binarization learning rate & $1 / n_{\text{bits}}$ (relative) \\
 & Binarization temperature & 0.001 \\
 & Binarization temperature (Softplus) & 0.001 \\
 & Residual initialization probability & 0.95 \\
\midrule
\multirow{3}{*}{\textbf{Learnable Thresholding (CIFAR-10)}} 
 & Binarization learning rate & 0.02 \\
 & Binarization temperature & 0.0002 \\
 & Binarization temperature (Softplus) & 0.01 \\
\midrule
\multirow{3}{*}{\textbf{High-Input Convolutions}} 
 & Batch size & 128 \\
 & Parametrization temperature & 1.0 \\
 & Residual initialization probability & 0.95 \\
\bottomrule
\end{tabular}
\label{tab:hyperparameters}
\end{table}


\end{document}